\documentclass[10pt,twocolumn,letterpaper]{article}

\usepackage[pagenumbers]{cvpr} % To force page numbers, e.g. for an arXiv version

\usepackage{graphicx}
\usepackage{amsmath}
\usepackage{amssymb}
\usepackage{booktabs}
\usepackage[pagebackref,breaklinks,colorlinks]{hyperref}

\usepackage[capitalize]{cleveref}
\crefname{section}{Sec.}{Secs.}
\Crefname{section}{Section}{Sections}
\Crefname{table}{Table}{Tables}
\crefname{table}{Tab.}{Tabs.}

\def\cvprPaperID{11241} % *** Enter the CVPR Paper ID here
\def\confName{CVPR}
\def\confYear{2022}

\usepackage[T1]{fontenc}    % use 8-bit T1 fonts
\usepackage{amsfonts}       % blackboard math symbols
\usepackage{nicefrac}       % compact symbols for 1/2, etc.
\usepackage{microtype}      % microtypography
\usepackage{adjustbox}
 
\newcommand{\yn}{\mathbf{y}^{\left(n\right)}}
\newcommand{\dnc}{d_c^{\left(n\right)}}
\newcommand{\qnc}{q_c^{\left(n\right)}}
\newcommand{\dncp}{d_{c'}^{\left(n\right)}}
\newcommand{\Kn}{\mathcal{K}^{\left(n\right)}}
\newcommand{\ra}[1]{\renewcommand{\arraystretch}{#1}}

\usepackage{multirow} 
\usepackage{bbold}
\usepackage{makecell}

\usepackage{algorithm}% http://ctan.org/pkg/algorithms
\usepackage{algorithmic}
\usepackage{xcolor}
\usepackage{amsthm}
\newtheorem{proposition}[]{Proposition}

\title{A sampling-based approach for efficient clustering in large datasets}

\author{%
 Georgios Exarchakis\thanks{Corresponding Author.} \\
 IHU Strasbourg, France\\
 ICube, University of Strasbourg, CNRS, France\\
 Sorbonne Universit\'e, INSERM, CNRS, France\\
 Institut de la Vision, 17 rue Moreau, 
 F-75012 Paris, France \\
  \texttt{georgios.exarchakis@ihu-strasbourg.eu} 
   \and
 Omar Oubari\\
 Sorbonne Universit\'e, INSERM, CNRS, \\
 Institut de la Vision, 17 rue Moreau, 
 F-75012 Paris, France
   \and
 Gregor Lenz\\
 Sorbonne Universit\'e, INSERM, CNRS, \\
 Institut de la Vision, 17 rue Moreau, 
 F-75012 Paris, France
}

\begin{document}
\maketitle

\begin{abstract} 
We propose a simple and efficient clustering method for
high-dimensional data with a large number of clusters. 
Our algorithm achieves high-performance by evaluating distances of 
datapoints with a subset of the cluster centres. 
Our contribution is substantially more efficient than k-means as 
it does not require an all to all comparison of data points and clusters.
We show that the optimal solutions of our approximation are the same 
as in the exact solution. However, our approach is considerably more 
efficient at extracting these clusters compared to the state-of-the-art.
We compare our approximation with the exact k-means and alternative 
approximation approaches on a series of standardised clustering tasks. 
For the evaluation, we consider the algorithmic complexity, including
number of operations to convergence, and the stability of the results.
An efficient implementation of the algorithm is available \href{https://github.com/ooub/peregrine}{online}.
\end{abstract}

\section{Introduction}
\label{sec:intro}
Data clustering is an ubiquitous problem in Machine Learning literature.
One of the most popular approaches for clustering is the k-means algorithm. 
Due to the simplicity of the algorithm and relative efficiency of identifying clusters   
the algorithm has found use in a wide range of fields, e.g.
medicine, physics and computer vision among many others. 
With the increasingly large supply of data, computational efficiency improvements become all the more significant.
Recent advances in approximate inference methods have allowed for training algorithms with convergence performance that is sub-linear to the number of clusters \cite{forster2018can}.
This is typically achieved by avoiding the comparison between datapoints and clusters that lie far away in feature space.

Probabilistic data models that relate to the k-means algorithm can be found 
in Gaussian Mixture Models (GMM). 
In particular, L\"ucke and Forster~\cite{luecke2019kmeans} 
detail the relationship between k-means and a variational 
approximation of the Expectation Maximisation (EM) algorithm 
for isotropic GMMs. In most cases, the GMM-based formulation 
of the clustering problem provides higher likelihoods 
without introducing impractical constraints \cite{luecke2019kmeans}.
% }
%
% The clustering algorithm itself can 
% be associated to finding the maximum likelihood solutions of a 
% GMM under certain constraints.
Theoretical developments in convergence analysis of Gaussian Mixture Models (GMM)
\cite{xu2016global,moitra2010settling,daskalakis2014faster} concerning 
global and optimal convergence have sparked renewed interest in the field.
Novel training algorithms that aim for stability \cite{Kolouri_2018_CVPR, hosseini2019alternative} and increased efficiency \cite{hirschberger2019accelerated, forster2018can} are actively being
developed.
Markov Chain Monte Carlo (MCMC) methods have also been employed for fitting GMMs to 
account for input variation \cite{LiMCMC2019,yang1998gaussian,favaro2013mcmc}. 
% However, no work utilises an MCMC method to predict the optimal cluster for 
% a datapoint prior to computing the distance of the cluster with that datapoint.
% {\color{red} 
Tree-based methods have also been proposed for efficient inference in k-means \cite{hamerly2010making,newling2016fast}. However, tree-based methods are known to exhibit instabilities when faced with small perturbations in the data, due to their recurrent structure, and therefore are likely to produce substantially different clusters for small changes in the input.
% } 

In this work, we propose a method for efficient
% and stable
EM-based learning 
that uses a truncated approximation \cite{lucke2010expectation,lucke2018truncated} of the posterior in the E-step. 
To identify the truncated space, we draw samples from a proposal distribution 
that is based on the truncated subspace of the previous iteration and 
favours clusters near the optimal cluster of the previous truncated posterior.
Our algorithm integrates recent developments in initialisation methods \cite{bachem2016approximate,bachem2016fast} and can be applied on coresets \cite{har2004coresets,bachem2018scalable} to maintain comparable performance 
to state-of-the-art approaches. 

% Furthermore, we can apply our algorithm on coresets \cite{har2004coresets} to reduce 
% algorithmic complexity of the training even more. 

Truncated approximations have been used in the past for multiple-cause 
models \cite{exarchakis2017discrete,dai2013invariant,forster2017truncated,exarchakis2012ternary} 
to achieve efficient training in discrete latent variable models.  
Stochastic approximations on a truncated space \cite{lucke2018truncated,guiraud2018evolutionary} 
focusing on deep learning models have also been proposed.
We expect a stochastic approach to avoid well-known local optima issues \cite{jin2016local} related to EM-based learning
for GMMs.

Truncated approximations on clustering algorithms have only been attempted with deterministic approximation techniques \cite{hirschberger2019accelerated,forster2018can}. In fact,
our literature research shows that the vc-GMM \cite{hirschberger2019accelerated} sets  the state-of-the-art in terms of computational efficiency for a GMM with similar constraints to the ones we study. vc-GMM relies on a set of indices of datapoints that get assigned to the same cluster in order to identify similar clusters. The most similar cluster is identified at each step and immediately assigned to the truncated space of a datapoint to efficiently navigate to the optimal clustering solution. The approach, however, is deterministic and such methods typically exhibit results unstable to initialisation and frequently converge to local optima. It would be therefore prudent to explore stochastic analogues to these approximations.
The approach we take in this work largely resembles vc-GMM, however, we utilise a similarity matrix over the clusters in order to identify clusters with a higher probability of being near a 
datapoint without having to evaluate its distance to all clusters.
Estimating the similarity matrix relies on posterior approximations from earlier 
iterations and does not require excessive computation. We show that our stochastic approach improves over the performance of vc-GMM and we consider the definition and implementation of our algorithm to be significantly simpler.

In the numerical experiments section, we evaluate the performance 
of the algorithm and compare with relevant literature. 
% We focus on three types of tasks: clustering on artificial data, 
% clustering on popular datasets and feature extraction for classification on event-based data. 
In the artificial data section, we evaluate our algorithm 
in terms of extracting the ground truth and we compare it 
to k-means to observe an improved performance. In the real 
data clustering section, we apply the algorithm on four different datasets, namely, KDD, a protein homology 
dataset \cite{caruana2004kdd},
CIFAR-10, an image dataset \cite{krizhevsky2009learning}, 
SONG, a dataset of music metadata \cite{Bertin-Mahieux2011} 
and SUSY, a high energy physics dataset \cite{baldi2014searching}. 
We compare our algorithm to the state-of-the-art 
in terms of efficiency, stability and accurate cluster recovery.

Results show that our algorithm sets the state-of-the-art in terms of efficiency 
without compromising, and probably improving, stability. Our method can be applied on a wide variety of tasks while 
maintaining a competitive clustering performance.

\section{EM with sparsely sampled clusters for GMMs}
\label{sec:gmm}
To introduce the k-means algorithm in terms of an optimisation framework, we consider a Gaussian Mixture Model fitted with the Expectation Maximisation algorithm. The relationship between k-means and variational approximations to a GMM are detailed by L\"ucke and Forster~\cite{luecke2019kmeans}. Using the GMM-based formalisation, we can derive a novel clustering algorithm with same global optima as the original algorithm, and substantial computational efficiency benefits.

For a dataset of $N$  data points, $\mathcal{Y}=\left\{\mathbf{y}^{\left(1\right)},\ldots,\mathbf{y}^{\left(N\right)}\right\}$ we wish to identify, $M$, cluster centres $\mu_c, \forall c\in\{1,\ldots,M\}$. 
To that end, each datapoint $\yn$ is treated as  an  
instance of a random variable $Y$ that follows one of $M$ 
possible Gaussian distributions $p\left(Y\vert C=c;\theta\right)= \mathcal{N}\left(Y;\mu_c,\sigma \mathbb{1}\right)$ 
with a prior probability distribution $p\left(C\right)=\frac{1}{M}$, where $C$ takes values in $\left\{1\,\ldots,M\right\}$ 
and $\theta=\left\{\mu_{1:M}, \sigma\right\}$ denotes the set of model parameters. We can learn the optimal parameters, 
$\theta=\left\{\mu_{1,\dots,M}, \sigma\right\}$, by maximising the data log-likelihood 
$\mathcal{L}\left(\theta\right)\triangleq \log p\left(Y=\mathcal{Y} \vert\theta\right)$ using the EM algorithm. The EM algorithm optimises the variational lower bound to the log-likelihood:
\begin{align}
L\left(\mathcal{Y},\theta\right)
\triangleq
\sum_{n}
\sum_{c}p_c^{(n)}
\log p\left(C=c,Y=\mathbf{y}^{\left(n\right)}\vert\theta\right)& \nonumber \\   +
\sum_{n}\mathcal{H}\left(p_c^{(n)}\right) &
\label{eq:free_t} 
\\
=
\sum_n\sum_c p_c^{(n)}\log\frac{p\left(C=c\vert Y=\yn,\theta\right)}{p_c^{(n)}} & \nonumber \\
+\sum_n\log p\left(Y=\yn \vert \theta \right) &
\label{eq:free_p}
\end{align}
%
%
%
% \begin{align}
% L\left(\mathcal{Y},\theta\right)
% \triangleq
% \sum_{n}
% \sum_{c}p_c^{(n)}
% \log p\left(C=c,Y=\mathbf{y}^{\left(n\right)}\vert\theta\right)  +
% \sum_{n}\mathcal{H}\left(p_c^{(n)}\right) 
% \label{eq:free_t} 
% \\
% =\sum_n\sum_c p_c^{(n)}\log\frac{p\left(C=c\vert Y=\yn,\theta\right)}{p_c^{(n)}} 
% +\sum_n\log p\left(Y=\yn \vert \theta \right)
% \label{eq:free_p}
% \end{align}
%
%
where $\mathcal{H}\left(p_c^{(n)}\right)$ 
denotes the Shannon entropy of the distribution $p_c^{(n)}$.
The distribution $p_c^{(n)}$  is typically set to be the posterior distribution $p\left(
C=c\vert Y=\yn;\hat{\theta}\right)$, as it sets the first term of Eq.~\ref{eq:free_p} 
to $0$ and the variational lower bound to be equal to the log-likelihood\footnote{The first term 
of Eq. 2 is the negative KL-divergence between the distribution $p_c^{(n)}$ and the exact 
posterior, $p\left(C=c\vert Y=\yn,\theta\right)$.}.
% =\frac{\exp\left(-\dnc/2 \sigma^{2}\right)}{
%  \sum_{c'=1}^M\exp\left(-\dncp /2 \sigma^{2} \right)}$. 
%
%
% \begin{equation}
% \mathcal{L}\left(\theta\right) \triangleq \sum_{n=1}^N\log
% \sum_{c=1}^M p\left(Y=\yn \vert C=c;\theta\right)p\left(C=c\right)
% \end{equation} 
% \textcolor{red}{n sums fr/om 1 to N right?.}\\
Under the assumed model constraints each Gaussian distribution is defined as
$p\left(Y=\yn\vert C=c; \theta\right)=\left(2\pi\sigma^{2}\right)^{-\frac{D}{2}}
e^{-\frac{\dnc}{2\sigma^{2}}}$,
where
$\dnc =\left\Vert \yn-\mathbf{\mu}_{c}\right\Vert^2 $
is the squared euclidean distance between the datapoint $\yn$ and the mean of Gaussian indexed by $c$, 
and $D$ is the number of observed variables.

Exact EM is an iterative algorithm that optimises the likelihood by alternating between two s.pdf. 
The first step, E-step, is to identify the distribution $p_c^{(n)}$ that sets the lower bound in Eq.~\ref{eq:free_p} 
to be equal to the log-likelihood. That is $p_c^{(n)}$ has to be equal to the posterior in order to set the  KL-divergence in Eq.~\ref{eq:free_p} to be equal to $0$. For the Gaussian Mixture Model in this work that would be:
\begin{equation}
    p_c^{(n)} = 
    \exp\left(-\dnc/2 \sigma^{2}\right)
    /
    \sum_{c'=1}^{M}\exp\left(-\dncp /2 \sigma^{2} \right) 
    % \frac{\exp\left(-\dnc/2 \sigma^{2}\right)}{\sum_{c'=1}^{M}\exp\left(-\dncp /2 \sigma^{2} \right)} 
    \label{eq:post}
\end{equation}
Here we notice that Eq.~\ref{eq:post} is a softmax function which produces mostly $0$ values. In fact, for $\sigma^2\rightarrow0$ it is exactly equal to the maximum indicator function for the (negative) distances, i.e. it returns the value $1$ for the smallest distance and $0$ for all others, and is often considered as the probabilistic analogue of k-means\cite{luecke2019kmeans,kulis2011revisiting,broderick2013mad}.
The second step, M-step, amounts to  maximising Eq.~\ref{eq:free_t} with 
respect to $\theta$ using a gradient update as:
\begin{equation}
    \mathbf{\mu}_c
    =
    \frac{
    \sum_{n=1}^N p_c^{(n)} \mathbf{y}^{(n)}
    }
    {
    \sum_{n=1}^N p_c^{(n)}
    }
    \label{eq:mu_update}
\end{equation}
% \begin{equation}
        % 
    % ,\,\,\,\,\,\,\,\,\,\,\hfill
% \end{equation}
\begin{equation}
    \sigma^2 = \frac{1}{DN}\sum_{n=1}^N\sum_{c=1}^Mp_c^{(n)} \left\|\mathbf{y}^{(n)} - \mathbf{\mu}_c\right\|^2
    \label{eq:sigma_update}
\end{equation}
The EM algorithm iterates between the E-step and M-step until $\theta$ converges. Updating $\sigma^2$, as opposed to the hard-assignment produced by $\sigma^2\rightarrow0$ in k-means, increases the variational lower bound \cite{luecke2019kmeans} and offers a better and more efficient approximation of the log-likelihood.

The E-step requires estimating the differences between all clusters and 
all the datapoints. 
Thus, the complexity of the E-step is $\mathcal{O}\left(DNM\right)$ 
making it a very efficient algorithm. 
 Here, we focus on a method
to avoid estimating the softmax over all dimensions since it leads to redundant computation.
In order to avoid the dependency of the complexity on $M$, we use an approximation $\qnc$ of the posterior, $p_c^{(n)}$, over a subset $\Kn\subset\left\{1,\ldots,M\right\}$, with $\vert\Kn\vert=H$ as:
\begin{equation}
    \qnc = \frac{\exp\left(-\dnc/2 \sigma^{2}\right)}{\sum_{c'\in \Kn}\exp\left(-\dncp /2 \sigma^{2} \right)} \delta \left(c\in\Kn\right)
    \label{eq:approx_post}
\end{equation}
where $\delta\left(c\in\Kn\right)$ is the Kronecker delta. In other words, we assume that clusters outside $\Kn$ have a probability of $0$ for datapoint $\yn$, and therefore are not estimated. 
Using $\qnc$ instead of $p_c^{(n)}$, modifies the exact EM algorithm by not setting the 
KL-divergence to $0$  at the E-step. However, we can derive an algorithm an algorithm 
that monotonically increases the variational lower bound by identifying a $\qnc$ that decreases the KL-divergence at 
each E-step. 
% Updating $\Kn$ iteratively with only clusters that are closer to the datapoint $\yn$ than those in $\Kn$ will monotonically reduce the Kullback-Leibler divergence, $K\!L\left[q^{(n)}\Vert p^{(n)}\right]$, and approximate the posterior.

\begin{proposition}
Let $\Kn$ be a set of cluster indices, and 
$\mathcal{K'}^{(n)}
=
\mathcal{K}^{(n)} 
\setminus
\left\{ i \right\}
\cup
\left\{ j \right\}$, where $i\in \mathcal{K}^{(n)}$, $j \notin \mathcal{K}^{(n)} $. Then $K\!L[q^{(n)}\Vert p^{(n)}]<
K\!L[{q'}^{(n)}\Vert p^{(n)}]$ if and only if 
$d_i^{(n)}<d_j^{(n)}$.
\label{prop:KL}
\end{proposition}

\begin{proof}
Since all the Gaussians are equiprobable $d_i^{(n)}<d_j^{(n)}\Rightarrow p_i^{(n)}>p_j^{(n)}$. Note that $\lim_{x\rightarrow0}x\log x=0$. It follows that:
\begin{align}
K\!L[q^{(n)}\Vert p^{(n)}]<
K\!L[{q'}^{(n)}\Vert p^{(n)}] & \Leftrightarrow \nonumber\\
\sum_{c\in \mathcal{K}^{(n)}} q_c^{(n)} \log\frac{q_c^{(n)}}{p_c^{(n)}}<
\sum_{c\in \mathcal{K'}^{(n)}} {q'}_c^{(n)} \log\frac{{q'}_c^{(n)}}{p_c^{(n)}} & \Leftrightarrow   \nonumber\\
\sum_{c\in \mathcal{K}^{(n)}} q_c^{(n)} \log\frac{
p_c^{(n)}/ \sum_{c'\in \mathcal{K}^{(n)}} p_{c'}^{(n)}
}{p_c^{(n)}}< &       \nonumber\\
\sum_{c\in \mathcal{K'}^{(n)}} {q'}_c^{(n)}  \log
\frac{
p_c^{(n)}/\sum_{c'\in \mathcal{K'}^{(n)}}p_{c'}^{(n)}
}{p_c^{(n)}} & 
\Leftrightarrow                            \nonumber\\
\sum_{c\in \mathcal{K}^{(n)}} q_c^{(n)} \log\sum_{c'\in \mathcal{K}^{(n)}} p_{c'}^{(n)}> \nonumber\\
\sum_{c\in \mathcal{K'}^{(n)}} {q'}_c^{(n)} \log \sum_{c'\in \mathcal{K'}^{(n)}} p_{c'}^{(n)} & \Leftrightarrow \nonumber\\
\log\sum_{c'\in \mathcal{K}^{(n)}} p_{c'}^{(n)}>
\log\sum_{c'\in \mathcal{K'}^{(n)}} p_{c'}^{(n)} 
\Leftrightarrow 
p_{i}^{(n)}>
p_{j}^{(n)} &\nonumber
\end{align}
\end{proof}
% \begin{proof}
% Since all the Gaussians are equiprobable $d_i^{(n)}<d_j^{(n)}\Rightarrow p_i^{(n)}>p_j^{(n)}$. Note that $\lim_{x\rightarrow0}x\log x=0$. It follows that:
% \begin{align}
% K\!L[q^{(n)}\Vert p^{(n)}]<
% K\!L[{q'}^{(n)}\Vert p^{(n)}] & \Leftrightarrow \nonumber\\
% \sum_{c\in \mathcal{K}^{(n)}} q_c^{(n)} \log\frac{q_c^{(n)}}{p_c^{(n)}}<
% \sum_{c\in \mathcal{K'}^{(n)}} {q'}_c^{(n)} \log\frac{{q'}_c^{(n)}}{p_c^{(n)}} & \Leftrightarrow   \nonumber\\
% \sum_{c\in \mathcal{K}^{(n)}} q_c^{(n)} \log\frac{
% p_c^{(n)}/ \sum_{c'\in \mathcal{K}^{(n)}} p_{c'}^{(n)}
% }{p_c^{(n)}}< &       \nonumber\\
% \sum_{c\in \mathcal{K'}^{(n)}} {q'}_c^{(n)}  \log
% \frac{
% p_c^{(n)}/\sum_{c'\in \mathcal{K'}^{(n)}}p_{c'}^{(n)}
% }{p_c^{(n)}} & 
% \Leftrightarrow                            \nonumber\\
% \sum_{c\in \mathcal{K}^{(n)}} q_c^{(n)} \log\sum_{c'\in \mathcal{K}^{(n)}} p_{c'}^{(n)}> \nonumber\\
% \sum_{c\in \mathcal{K'}^{(n)}} {q'}_c^{(n)} \log \sum_{c'\in \mathcal{K'}^{(n)}} p_{c'}^{(n)} & \Leftrightarrow \nonumber\\
% \log\sum_{c'\in \mathcal{K}^{(n)}} p_{c'}^{(n)}>
% \log\sum_{c'\in \mathcal{K'}^{(n)}} p_{c'}^{(n)} 
% \Leftrightarrow 
% p_{i}^{(n)}>
% p_{j}^{(n)} &\nonumber
% \end{align}
% %
% \end{proof}

Proposition \ref{prop:KL} shows that in order to decrease the KL-divergence at each E-step 
we only need to iteratively update  the set $\Kn$ with clusters that 
are closer to the data-point, $\yn$. The M-step can be modified to utilise $\qnc$, instead of $p_c^{(n)}$, and maintain monotonic convergence %to a local optimum
\cite{NealHinton1998}.

To identify the clusters in $\Kn$, we start by selecting  $H$ clusters uniformly 
at random. We iteratively update $\Kn$ by using $R$ randomly sampled  clusters 
in the vicinity of the one that is nearest to the datapoint $\yn$. 
To efficiently identify the clusters centred near a datapoint, we define a distribution 
$p\left(C_t\vert C_{t-1}=\bar{c}_n; S\right)$, 
where  $\bar{c}_n={\arg\min}_c \left\{\dnc\vert c\in \Kn\right\}$. 
The parameter $S\in\mathbb{R}^{M\times M}$ denotes a similarity matrix among the clusters that assigns higher values, $S_{i,j}$, to cluster pairs, $\{i,j\}$, that are likely to be close to the same datapoints, as in Eq. \ref{eq:sim_b}. 
The iterative update of $\Kn$ is defined as:
\begin{align}
    &\bar{\mathcal{K}}^{(n)}_t=
    %   \left\{  
    %   c_{1:H}\vert c_i \in\left[
    \Kn_{t-1} 
      \cup  \left\{c_{1:R}\vert c_i \sim p\left(C_t\vert C_{t-1}=\bar{c}_n\right)\land c_i \notin \Kn_{t-1} \right\}
      \\
    &\Kn_t=\left\{ c\vert c\in \bar{\mathcal{K}}^{(n)}_t \textrm{ with the $H$ smallest }d_{c}^{(n)}\right\} 
    %   \right  ]_{<_{\dnc}}
    %   \right\} 
    \label{eq:Kn-it-prior}
\end{align}
where $t$ denotes the EM iteration. $p\left(C_t\vert C_{t-1}=\bar{c}_n; S\right)$ is the distribution that 
is given by the normalised row of a cluster similarity matrix $S$ after setting the probabilities corresponding to $\Kn$ to $0$:
\begin{equation}
    p\left(C_t=c\vert C_{t-1}=\bar{c}_n; S\right) = \frac{S_{\bar{c}_n,c}}{\sum_{c'\in\Kn}S_{\bar{c}_n,c'}}
    \delta\left(c\notin \Kn\right)
\end{equation}
i.e. the distribution at time $t$ is given by the row defined by the cluster $\bar{c}_n$ that had the minimal distance with the datapoint $\yn$ at time $t-1$.

The parameter updates, from Eq. \ref{eq:mu_update} and \ref{eq:sigma_update}, are adapted to the approximate posterior. 
% Additionally, we can define incremental updates to avoid loading the entire dataset in memory at any one time. 
% For a subset of the dataset $\left\{\yn\right\}_{\forall n \in \mathcal{B}}$, with  $\mathcal{B}\subset\mathcal{J}=\left\{1,\ldots,N\right\}$ the updates become:
% the memory requirements by optimising the parameters $\theta$ over a subset of the dataset $\left\{\yn\right\}_{\forall n \in \mathcal{B}}$, with  $\mathcal{B}\subset\mathcal{J}=\left\{1,\ldots,N\right\}$, at each iteration step. Thus the updates in Eqs. \ref{eq:mu_update} and \ref{eq:sigma_update} become:
\begin{equation}
    \mathbf{\mu}_c%^{\mathcal{B}}
    =
    \sum_{n=1}^N \qnc \mathbf{y}^{(n)}
    /
    \sum_{n=1}^N \qnc 
    \label{eq:mu_update_b}
\end{equation}

\begin{equation}
    \sigma^2%_{\mathcal{B}} 
    = 
    \frac{1}{DN}\sum_{n=1 }^N\sum_{c\in\Kn} \qnc \left\|\mathbf{y}^{(n)} - \mathbf{\mu}_c\right\|^2
    \label{eq:sigma_update_b}
\end{equation}
% \begin{equation}
%     \mathbf{\mu}_c%^{\mathcal{B}}
%     =
%     \sum_{n=1}^N \qnc \mathbf{y}^{(n)}
%     /
%     \sum_{n=1}^N \qnc 
%     % \label{eq:mu_update_b}
%     ,\,\,\,\,
%     \sigma^2%_{\mathcal{B}} 
%     = 
%     \frac{1}{DN}\sum_{n=1 }^N\sum_{c\in\Kn} \qnc \left\|\mathbf{y}^{(n)} - \mathbf{\mu}_c\right\|^2
%     \label{eq:sigma_update_b}
% \end{equation}

The similarity matrix is defined based on the 
distances $\dnc$ under the assumption that nearby clusters have small distances to similar datapoints 
\begin{equation}
    S_{i,j} = \frac{1}{N} \sum_{n =1}^N e^{-\left(d_i^{(n)}+d_j^{(n)}\right)}
    \delta\left(\{i,j\} \subset\Kn\right)
    \label{eq:sim_b}
\end{equation}
\begin{algorithm}[H]
% {\scriptsize
\begin{algorithmic}[1]
\REQUIRE {Dataset $\mathcal{X}$, \# of centres $M$} 
\STATE   initialise $\mathbf{\mu}_{1:M}$, $\sigma$, $\mathcal{K}^{\left(n\right)}$ and $S=0$
for all $n$;
\REPEAT
\STATE $\mu_{1:M}^{new}=0$, $\sigma^{new}=0$, and $S^{new}=0$
\STATE $\mathcal{J}=\left\{1,\ldots,N\right\} $ 
% \FOR{$\mathcal{B} \subset \mathcal{J} $}
%%%
% {\bf M-Step}
%
\FOR{$n \in \mathcal{J}$}
\STATE $\bar{c}_n={\arg\min}_c \left\{\Vert \yn - \mu_c\Vert^2\vert c\in \Kn\right\}$
\STATE $p\left(C_t=c\vert C_{t-1}=\bar{c}_n;S\right):=\frac{S_{\bar{c}_n,c}}{\sum_{c'}S_{\bar{c}_n,c'}}$
% \STATE $p\left(C_t\vert C_{t-1}=\bar{c}_n;S\right):=\begin{cases}\frac{\bf{S}_{\bar{c}_n}}{\sum_{c'}S_{\bar{c}_n,c'}} & \Kn \neq \emptyset\\ \frac{1}{M}&  \Kn=\emptyset \end{cases}$
\STATE $ \bar{\mathcal{K}}^{(n)}= \left\{c_{1:R}\vert c_i \sim p\left(C_{t} \vert C_{t-1}=\bar{c}_n;S\right)\land  c_i\notin\Kn \right\}$
\STATE $ \bar{\mathcal{K}}^{(n)}=\bar{\mathcal{K}}^{(n)} \cup  \Kn$
\FOR{$c\in \bar{\mathcal{K}}^{(n)}$}
\STATE  $d_{c}^{(n)}=\left\Vert \mathbf{y}^{(n)}-\mathbf{\mu}_{c}\right\Vert^2 $
\ENDFOR
\STATE  $\Kn=\left\{ c\vert c\in \bar{\mathcal{K}}^{(n)} \textrm{ with the $H$ smallest }d_{c}^{(n)}\right\} $
\ENDFOR
\STATE Calculate $\mu_{1:M}$,$\sigma^2$, and $S$ using  Eqs. \ref{eq:mu_update_b}--\ref{eq:sim_b}
\UNTIL  $\mathbf{\mu}_{1:M}$ and $\sigma^{2}$ have converged 
\end{algorithmic}\caption{Data Similarity Gaussian Mixture Model (D-GMM)}\label{alg:st-gmm}
% }
\end{algorithm}
Eq. \ref{eq:sim_b} produces a symmetric positive definite matrix,
that is used to sample datapoints near the optimal at each step of the process, 
with a simple reduction operation over pre-computed values.
Iterating between Eq. \ref{eq:approx_post} and 
the parameter updates, Eqs. \ref{eq:mu_update_b}-\ref{eq:sim_b}, details an algorithm that 
we call Data Similarity GMM (D-GMM), Alg. \ref{alg:st-gmm}, due to the similarity matrix being based on a data ``voting'' process.
% The full details of the algorithm can be seen in  Algorithm \ref{alg:st-gmm}.
%
The complexity of an E-step of the  \emph{D-GMM} algorithm reduces compared to an E-step of the  exact EM algorithm for GMMs from $\mathcal{O}\left(NMD\right)$ 
to  $\mathcal{O}\left(N\left(R+H\right)D\right)$, where typically $R+H<<M$. 
For the M-step, the complexity becomes $\mathcal{O}\left(NHD+NH^2\right)$ from $\mathcal{O}\left(NMD\right)$, however, as we will show in the experiments' section, $H^2<<M$ to be sufficient for most applications.

\paragraph{Initialisation.} 
During the first epoch of  the proposed  algorithm the sets $\Kn$ are initialised using prior samples.
The centres of the gaussians, $\mu_{1:C}$, are initialised using the 
AFK-MC\textsuperscript{2} \cite{bachem2016fast} initialisation method.  After an epoch has passed, the $\Kn$ is updated as in algorithm \ref{alg:st-gmm}.
% sampled clusters.
%
The AFK-MC\textsuperscript{2} algorithm samples an initial centre $\mu_1\in \mathcal{Y}$ uniformly at random and then uses it to derive the 
proposal distribution $g(\mathbf{y}|\mu_1)$. A markov chain of length $m$ is used to sample sufficiently distinct new centres, $\mu_{2:M}$, iteratively. 
% A detailed description of the algorithm can be seen in algorithm \ref{alg:afk-mc2}.
% \begin{algorithm}[H]
% \begin{algorithmic}[1]
% \REQUIRE {Data set $\mathcal{Y}$, \# of centres $M$, chain length $m$} 
% \STATE    $\mu_1 \leftarrow $ Point uniformly sampled from $\mathcal{Y}$
% for all $n$;
% \FORALL{$\mathbf{y}\in\mathcal{Y}$}
% \STATE $g\left(\mathbf{y}\right)\leftarrow \frac{1}{2} \left\|\mathbf{y}-\mu_1\right\|^2/\sum_{\mathbf{y}' \in \mathcal{Y}} \left\|\mathbf{y}'-\mu_1\right\|^2 + \frac{1}{2n}$
% % \STATE $\mu_1 \leftarrow \{c_1\}$
% \ENDFOR
% \FOR{$c=2,3,\ldots,M$}
% \STATE $\mathbf{y}\leftarrow$ Point sampled from $\mathcal{Y}$ using $g\left(\mathbf{y}\right)$
% \STATE $d_{\mathbf{y}}\leftarrow \min \left\|\mathbf{y}-\{\mu_{1:c-1}\}\right\|^2 $
% \FOR{$l=2,3,\ldots,m$}
% \STATE $\mathbf{x}\leftarrow$ Point sampled from $\mathcal{Y}$ using $g\left(\mathbf{x}\right)$
% \STATE $d_{\mathbf{x}}\leftarrow \min \left\|\mathbf{x}-\{\mu_{1:c-1}\}\right\|^2$
% \IF{$\frac{d_{\mathbf{x}}g\left(\mathbf{x}\right)}{d_{\mathbf{y}}g\left(\mathbf{y}\right)} > \textrm{Unif}\left(0,1\right)$}
% \STATE $\mathbf{y}\leftarrow\mathbf{x}$, $d_{\mathbf{y}}\leftarrow d_{\mathbf{x}}$
% \ENDIF
% \ENDFOR
% \STATE $\mu_{c} \leftarrow \mathbf{y}$
% \ENDFOR
% \end{algorithmic}\caption{AFK-MC\textsuperscript{2}}\label{alg:afk-mc2}
% \end{algorithm}
%
%
The complexity of  AFK-MC\textsuperscript{2} is 
$\mathcal{O}\left(N D\right)$  to define the proposal distribution $g\left(\mathbf{y}\vert \mu_{1}\right)$. The centres are sampled from the data using Markov chains of length $m$ with a complexity of  $\mathcal{O}\left(m(M-1)^2 D\right)$.
\paragraph{Lightweight Coresets (lwcs).} To further improve computational efficiency we can optionally use coresets of the dataset \cite{bachem2018scalable,feldman2011scalable,lucic2017training}. 
Coresets are smaller, $N'<<N$, representative subsets, $\mathcal{Y}'=\left\{\left(\mathbf{y}_1,w_1\right),\ldots,\left(\mathbf{y}_{N'},w_{N'}\right) \right\}$,
of a full dataset, $\mathcal{Y}$, in which each datapoint is individually weighted by a weight $w_{1:N'}$ depending on its significance in describing the original data. The objective on a coreset is adjusted to account for the weights on each data point:
\begin{align}
L\left(\mathcal{Y}',\theta\right)
\triangleq
\sum_{n} w_{n}
\sum_{c\in \Kn}q_c^{(n)}
\log p\left(C=c,Y=\mathbf{y}^{\left(n\right)}\vert\theta\right)   
\nonumber\\
+
\sum_{n}w_{n}\mathcal{H}\left(\qnc\right)
\label{eq:free_c}
\end{align}

Since the parameter updates are gradient-based updates of Eq. \ref{eq:free_c}, the weights $w_{1:N'}$ are a multiplicative constant on the parameters and therefore the parameter updates become: 
\begin{equation}
    \mathbf{\mu}_c
    =
    \sum_{n=1}^{N'} w_{n} \qnc \mathbf{y}^{(n)} /
    \sum_{n=1}^{N'} w_{n} \qnc 
    \label{eq:mu_update_b_c} 
    % \textrm{, \,\,\,\,}
\end{equation}
% \begin{equation}
%     q_c^{\mathcal{B}}
%     =
%     \sum_{n\in\mathcal{B}} w_{n} \qnc 
%     \label{eq:norm_update_b_c}
% \end{equation}
\begin{equation}
    \sigma^2 = \frac{1}{DN'}\sum_{n =1}^{N'}\sum_{c} w_n \qnc \left\|\mathbf{y}^{(n)} - \mathbf{\mu}_c\right\|^2
    \label{eq:sigma_update_b_c}
\end{equation}
\begin{equation}
    % \textrm{, \,\,\,\,}
    S_{i,j} = \frac{1}{N'} \sum_{n=1}^{N'} w_n e^{-\left(d_i^{(n)}+d_j^{(n)}\right)}
    \delta\left(\{i,j\} \subset\Kn\right)
\end{equation}
\begin{figure}[h]
    \centering
    \includegraphics[width=0.45\textwidth]{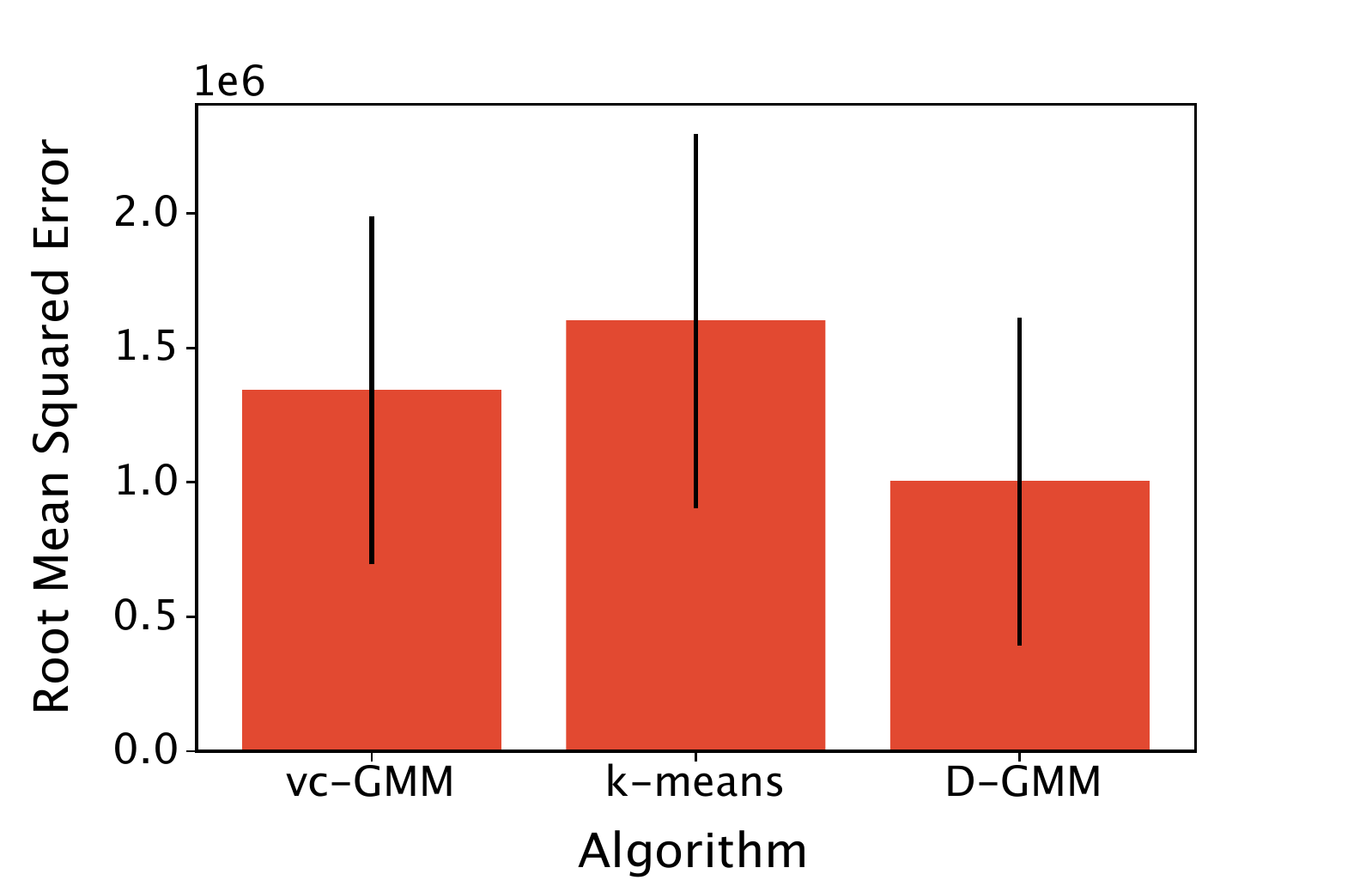}
    \includegraphics[width=0.42\textwidth]{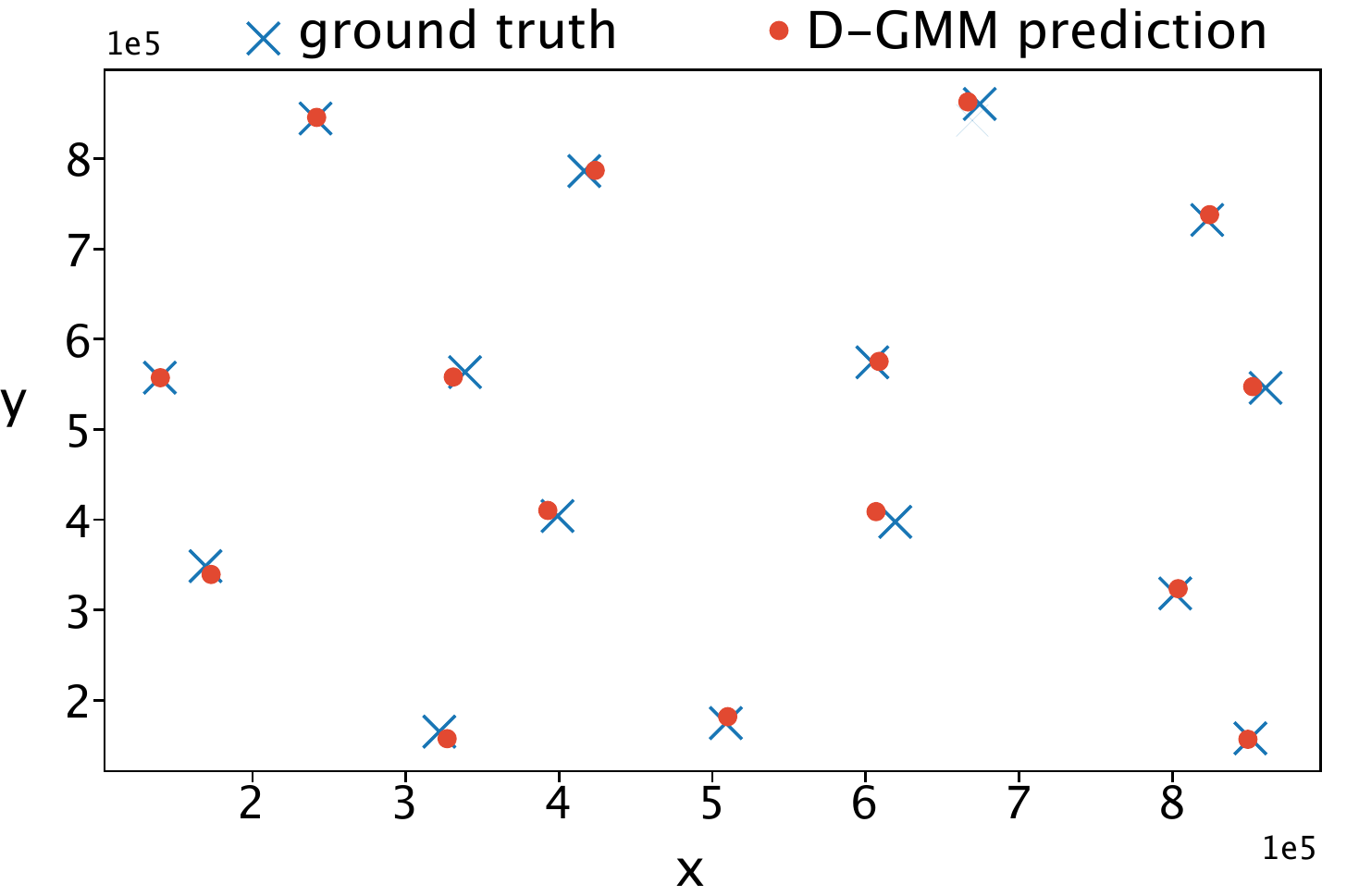}
    \caption{{\bf Validation on a two-dimensional synthetic dataset:} D-GMM outperforms other methods across $500$ trials (top) and fully recovers all centres (bottom)}
    \label{fig:articifial}
\end{figure}
These updates can replace Eq. \ref{eq:mu_update_b} to \ref{eq:sim_b} 
in algorithm \ref{alg:st-gmm} to allow applications on a coreset $\mathcal{Y}'$. Working 
on coresets introduces an error in the approximation that has been analysed rigorously in earlier work \cite{bachem2018scalable}. Constructing the coreset requires two iterations of complexity 
$\mathcal{O}\left(ND\right)$ over the data. Working on coreset reduces the complexity of D-GMM 
to $\mathcal{O}\left(N'\left(R+H\right)D\right)$ for the E-step and $\mathcal{O}\left(N'HD + N'H^2\right)$ for the M-step. The complexity of AFK-MC\textsuperscript{2} 
is also reduced since the proposal distribution is defined on the coreset with complexity $\mathcal{O}\left(N'D\right)$.

\section{Experiments and results}
\label{sec:exp}

We evaluate the performance of the algorithm experimentally on three classes of tasks. The software used for the experiments is a vectorised C++ implementation provided in the supplementary material. 
First, we examine convergence on artificial data where the ground truth is known. 
We proceed with a comparison against the  state-of-the-art algorithm for training GMMs with similar constraints vc-GMM \cite{hirschberger2019accelerated} on popular clustering 
datasets. 
% The main result of our work is out ability to simplify this algorithm using a stock
% Lastly, to demonstrate the emerging necessity of Machine Learning algorithms developed 
% with high efficiency requirements, we present an application to a dataset
% extracted from event-driven visions sensors. In this regime standard k-means 
% becomes extremely slow and we resort to efficient variants of the standard algorithm

\begin{figure}[ht!]
\includegraphics[width=0.4\textwidth]{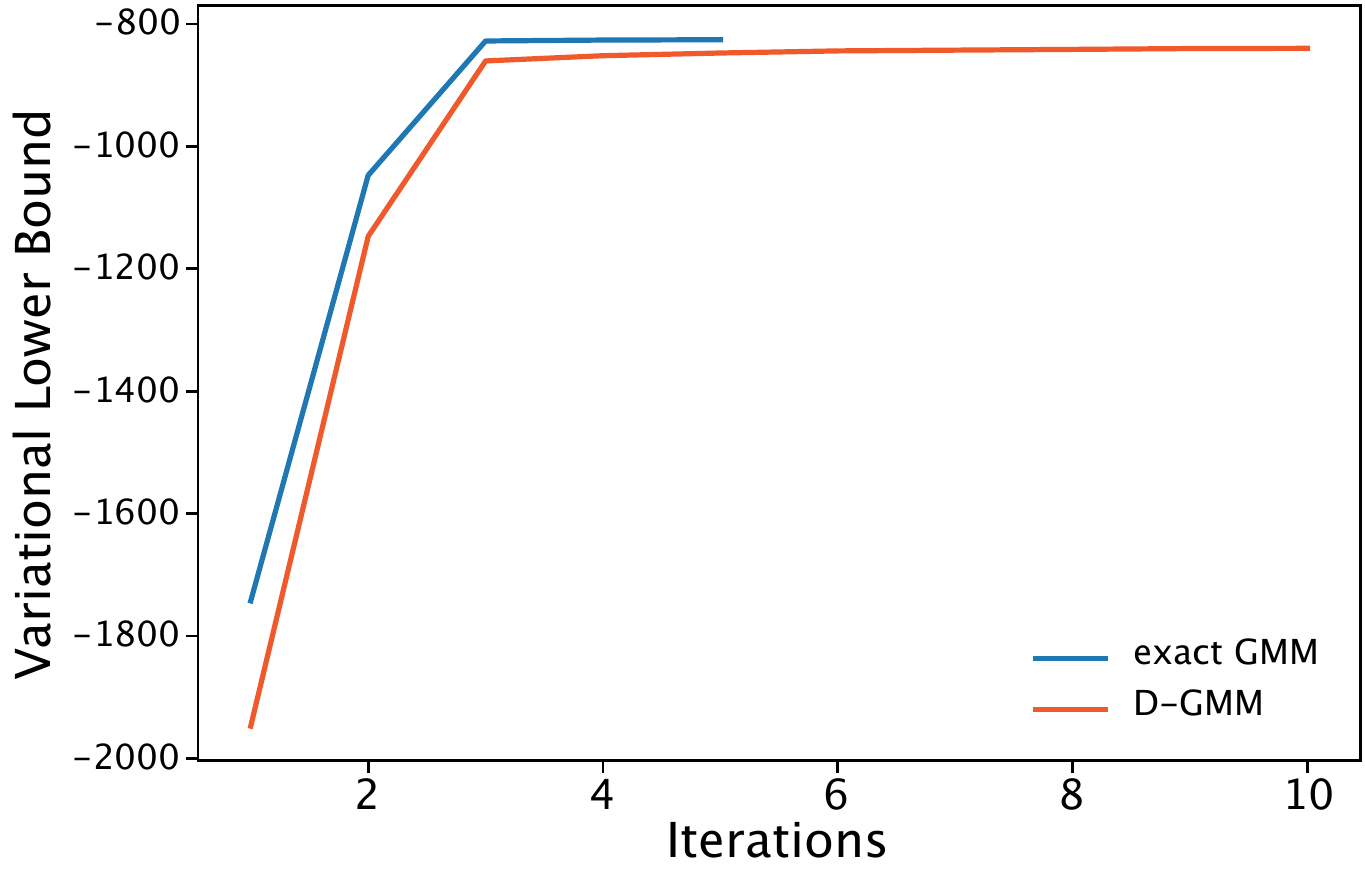}
\caption{\textbf{Variational lower bound values of D-GMM through EM iterations compared to the log-likelihood of the exact algorithm:} We can see that our approximation
is slower to converge, however, each iteration is considerably more efficient.
}
\label{fig:free_energy}
\end{figure}

For all tasks, the Gaussian centres are initialised using the AFK-MC\textsuperscript{2} algorithm with $m=5$. Furthermore, we follow the same \textbf{convergence protocol} as in \cite{hirschberger2019accelerated} and terminate the algorithm 
when the variational lower bound increment following Eq. \ref{eq:free_c} is less than $.pdfilon=10^{-3}$.
Unless stated otherwise, we evaluate the stability of the 
results on 10 repetitions for all experiments.

For clarity, below is a reminder for the hyperparameter notations:

\begin{itemize}
\itemsep0em 
    \item $M$ denotes the number of centres
    \item $N'$ denotes the coreset size
    \item $H$ and $C'$ denote the size of the truncated subspace for D-GMM and vc-GMM respectively.
    \item $R$ and $G$ are the search space hyperparameters for D-GMM and vc-GMM respectively.  
\end{itemize}

When choosing the truncation hyperparameters $H$ ($C'$), we consider that the probability values of the exact posterior decays exponentially and accordingly set $H=5$ ($C'=5$) under the assumption that
lower probability values will be negligible. We follow the same rationale for the truncation updates $R$ ($G$). We use various configurations for $M$ and $N'$ so we can compare with the state-of-the-art.

\subsection{Artificial Data}

In this section, we present a convergence analysis on artificial data \cite{Ssets} with $N=5000$ data points and $15$ Gaussian centres.
Fig. \ref{fig:articifial} on the left shows the root mean squared error between the learned centres of the algorithms and the ground truth centres. We compare our algorithm, D-GMM, against vc-GMM, and standard k-means, setting the hyperparameters to $M=15$, $N'=1000$, $H=3$ and $R=5$. The vc-GMM is parametrised with $C'=3$ and $G=5$.
% global optimum figure
The results suggest that both truncated algorithms are able to recover the centres as well as the exact algorithm.
The slight improvement (below a standard deviation) might be attributed to the fact that a truncated approximation
will ``hard--code'' very low probabilities to $0$ which may enhance numerical stability. With the D-GMM algorithm, the stochastic behaviour might also have an effect on avoiding locally optimal solutions. In Fig. \ref{fig:articifial} on the right, we present an example of a run where the centres were successfully recovered.

\begin{table*}[ht!]
\centering
\footnotesize
\caption{\textbf{Relative quantisation error and distance evaluation speedup}}
\ra{1.3}
\begin{adjustbox}{max width=0.9\textwidth}
\begin{tabular}{@{}c c cccc c c c@{}}\toprule
\textbf{Dataset} & \multicolumn{4}{c}{\textbf{Algorithm}} & \textbf{\makecell{Relative \\Error $\eta$}} & \textbf{\makecell{Distance Eval. \\Speedup}} & \textbf{\makecell{Iters.}}\\
\cmidrule{2-5}
& name & $N'$ & $H(C')$ & $R(G)$\\\midrule
\multirow{3}[2]{*}{\makecell{\textbf{KDD}\\\scriptsize $N=145,751$\\ \scriptsize$D=74$ \\ \scriptsize$M=500$}}
                             & k-means & -        & -  & -   & $0.0\pm0.7\%$   & $\times1.0\pm0.0$      & $5.0\pm0.0$\\
                             & k-means + lwcs & $2^{12}$ & - & -    & $14.0\pm0.5\%$  & $\times35.1\pm0.0$ & $5.0\pm0.0$ \\
                             & vc-GMM  & $2^{12}$ & 5 & 5    & $12.0\pm0.4\%$  & $\times533.1\pm36.5$ & $11.7\pm1.0$\\
                             & \textbf{D-GMM}   & $\mathbf{2^{12}}$ & \textbf{5} & \textbf{5}   & $\mathbf{12.0\pm1.0\%}$  & $\mathbf{\times622.1\pm28.0}$ & $\mathbf{17.0\pm0.7}$\\
\hline                 
\multirow{4}[2]{*}{\makecell{\textbf{CIFAR-10}\\\scriptsize $N_{\text{train}}=50,000$\\ \scriptsize$N_{\text{test}}=10,000$\\ \scriptsize$D=3,072$ \\ \scriptsize$M=500$}}
                             & k-means  & -        & - & -  & $0.0\pm0.0\%$ & $\times1.0\pm0.0$ & $7.6\pm0.4$\\
                             & k-means + lwcs  & $2^{12}$ & - & -  & $7.0\pm0.1\%$ & $\times48.9\pm3.5$ & $5.4\pm0.4$\\
                             & vc-GMM   & $2^{12}$ & 5 & 5  & $7.0\pm0.0\%$ & $\times674.7\pm45.8$ & $11.8\pm0.8$\\
                             & \textbf{D-GMM}    & $\mathbf{2^{12}}$ & \textbf{5} & \textbf{5} & $\mathbf{8.0\pm0.0\%}$ & $\mathbf{\times731.5\pm41.9}$ & $\mathbf{21.4\pm1.2}$\\
\hline               
\multirow{3}[2]{*}{\makecell{\textbf{SONG}\\\scriptsize $N=515,345$\\ \scriptsize$D=90$ \\ \scriptsize$M=4000$}}   
                             & k-means  & -        & - & -  & $0.0\pm0.0\%$ & $\times1.0\pm0.0$ & $5.0\pm0.0$\\
                             & k-means + lwcs  & $2^{16}$ & - & -  & $8.0\pm0.0\%$ & $\times7.8\pm0.0$ & $5.0\pm0.0$\\
                             & vc-GMM   & $2^{16}$ & 5 & 5  & $8.0\pm0.1\%$ & $\times698.2\pm0.7$ & $12.0\pm0.0$\\
                             & \textbf{D-GMM}    & $\mathbf{2^{16}}$ & \textbf{5} & \textbf{5} & $\mathbf{8.0\pm0.2\%}$ & $\mathbf{\times862.1\pm18.3}$ & $\mathbf{21.7\pm0.4}$\\
\hline
\multirow{3}[2]{*}{\makecell{\textbf{SUSY}\\\scriptsize $N=5,000,000$\\ \scriptsize$D=18$ \\ \scriptsize$M=2000$}}    
                             & k-means  & -        & - & -  & $0.0\pm0.0\%$  & $\times1.0\pm0.0$ & $14.7\pm0.4$\\
                             & k-means + lwcs  & $2^{16}$ & - & -  & $6.0\pm0.1\%$ & $\times11.1\pm0.4$ & $14.1\pm0.5$\\
                             & vc-GMM   & $2^{16}$ & 5 & 5  & $6.0\pm0.1\%$ & $\times663.1\pm17.1$ & $25.4\pm0.6$\\
                             & \textbf{D-GMM}    & $\mathbf{2^{16}}$ & \textbf{5} & \textbf{5} & $\mathbf{5.0\pm0.1\%}$  & $\mathbf{\times605.7\pm11.1}$ & $\mathbf{55.6\pm1.0}$\\
\bottomrule
\end{tabular}
\end{adjustbox}
\label{tab:clustering_analysis}
\end{table*}

\subsection{Clustering Analysis}
For a more detailed comparison with the state-of-the-art, we consider a series of well-known clustering 
datasets. Tab. \ref{tab:clustering_analysis} details a comparison between k-means, vc-GMM 
\cite{hirschberger2019accelerated,forster2018can}, and D-GMM. We use the 
k-means algorithm on the full dataset to define a baseline for the centres. The accuracy of the rest of the algorithms
is measured using the relative error 
$\eta = \left(\mathcal{Q}_{\textrm{algorithm}}-\mathcal{Q}_{\textrm{k--means}}\right)/ \mathcal{Q}_{\textrm{k--means}}$, where $\mathcal{Q}$ stands for an algorithm's quantization error. 
Since D-GMM and vc-GMM are going through fewer clusters per datapoint in each iteration,  convergence is slower
for these two algorithms (see Fig.~\ref{fig:free_energy}). However, the efficiency of the algorithm is determined by the overall number of 
distance evaluations. In the last two columns of Tab. \ref{tab:clustering_analysis}, we present the average
number of iterations from initialisation to convergence as well as the average speedup in terms of distance evaluations, $\dnc$, 
relative to k-means. 
% $N'$ denotes the coreset size, H and C' the size of the truncated subspace 
% for  D--GMM and vc-GMM respectively, and R and G are the search space hyperparameters for D-GMM and vc-GMM respectively. 
The results show a clear speedup for D-GMM in most cases and comparable relative error.
% We also present the factor of improvement compared to k-means in terms of distance evaluation, $\dnc$,
% relative to k-means. 
% operations complexity figure
% comparison table clustering analysis

%
Fig. \ref{fig:clustering_analysis} presents a comparison between the efficient clustering algorithms 
with increasing coreset size. K-means on the full dataset is presented as baseline. The size of each 
marker in Fig. \ref{fig:clustering_analysis} represents the size of the coreset. We find that in most 
cases D-GMM clusters data with a low relative error to the baseline for the least 
amount of distance evaluations. 
%
% \begin{adjustbox}{max width=0.9\textwidth}
\begin{figure*}[ht!]%[htp]
    \centering
    \includegraphics[width=0.8\textwidth]{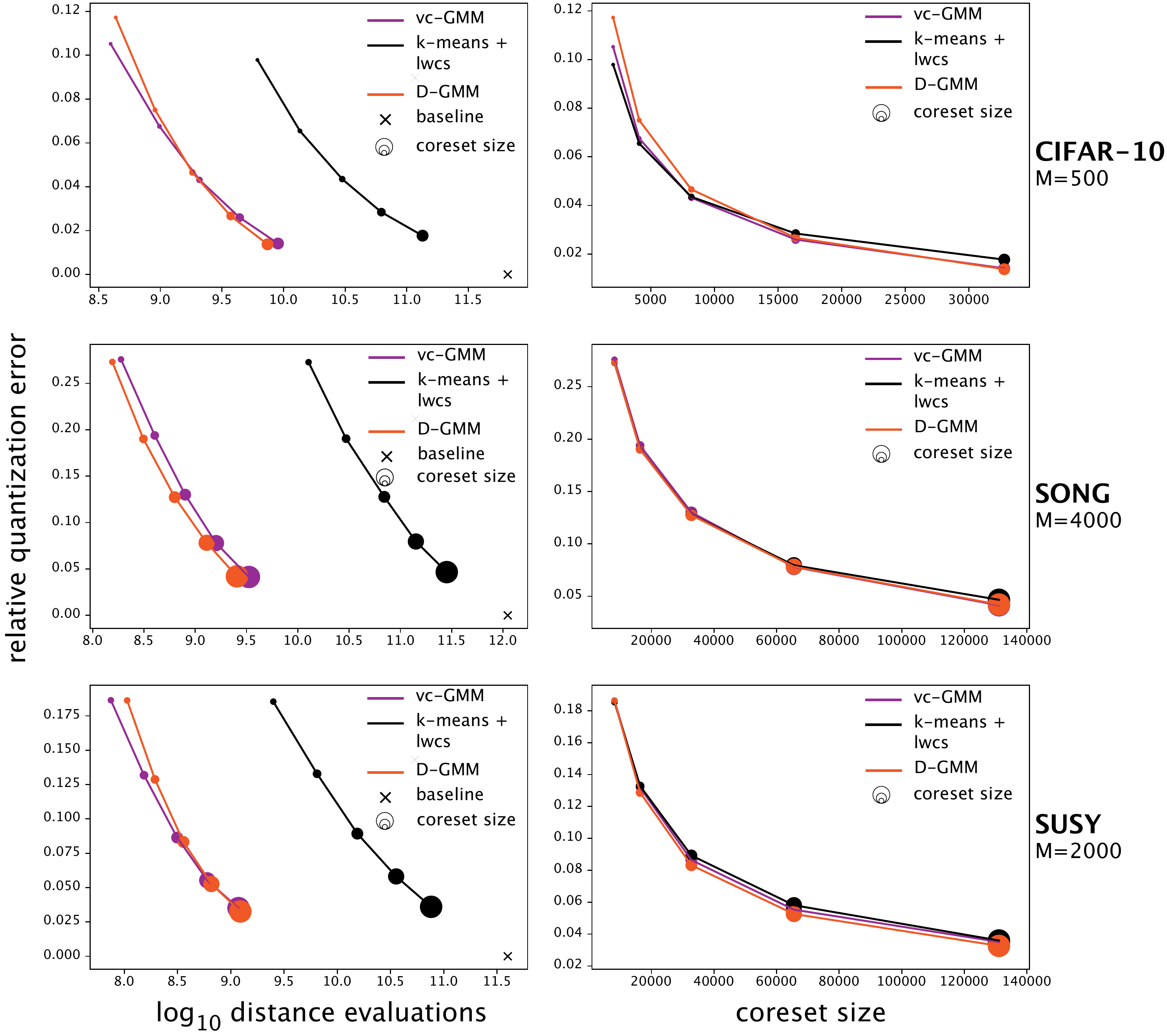}
    \caption{{\bf Distance evaluations vs relative quantisation error on increasingly large datasets}} 
    \label{fig:clustering_analysis}
\end{figure*}
% \end{adjustbox}
%
\paragraph{Complexity} The approximation method we use is focused on avoiding distance
evaluations, $\dnc$ with all available clusters. Therefore, it is very efficient in problems where 
a high number of clusters is expected to be present in the dataset. Fig. \ref{fig:complexity} (left) 
shows the scaling behaviour of our algorithm with an increasing number of clusters $M$ on 
the CIFAR-10 dataset, with $M$ ranging from $100$ to $1500$ cluster centres. The 
distance evaluations for each algorithm are normalised by the minimum value across all $M$
% results of each algorithm
% are normalised by the
% smallest number of distance evaluations 
and presented in a log-log plot which indicates the power of the relationship between operation complexity and number of clusters. We normalised both axes for an easier visualisation of the complexity. As expected, the scaling behaviour of k-means is linear to the number of clusters while the approximations are sub-linear. D-GMM is the most efficient algorithm in terms of distance evaluations as the number of cluster centres increases.
%

% \begin{table}\centering
% \caption{\textbf{Classification accuracy on very large event-based datasets}}
% \begin{adjustbox}{max width=0.48\textwidth}
% \begin{tabular}{@{}c c c c@{}}\toprule
% & \makecell{\textbf{N-CARS}\\$N'=2^{16}$} & \makecell{\textbf{N-MNIST}\\$N'=2^{16}$}\\ \midrule
% \textbf{D-GMM}  & $\mathbf{73.25\pm0.68\%}$ & $\mathbf{98.25\pm0.00\%}$ \\
% vc-GMM &  $72.93\pm0.40\%$ & $98.16\pm0.04\%$\\
% efficient k-means & $73.26\pm0.48\%$ &  - \\
% SLAYER \cite{shrestha2018slayer} & - & $99.2\pm0.02\%$\\
% \bottomrule
% \end{tabular}
% \end{adjustbox}
% \label{tab:classification_results}
% \end{table}
\begin{figure*}[h!]
    \centering
    \includegraphics[width=0.38\textwidth]{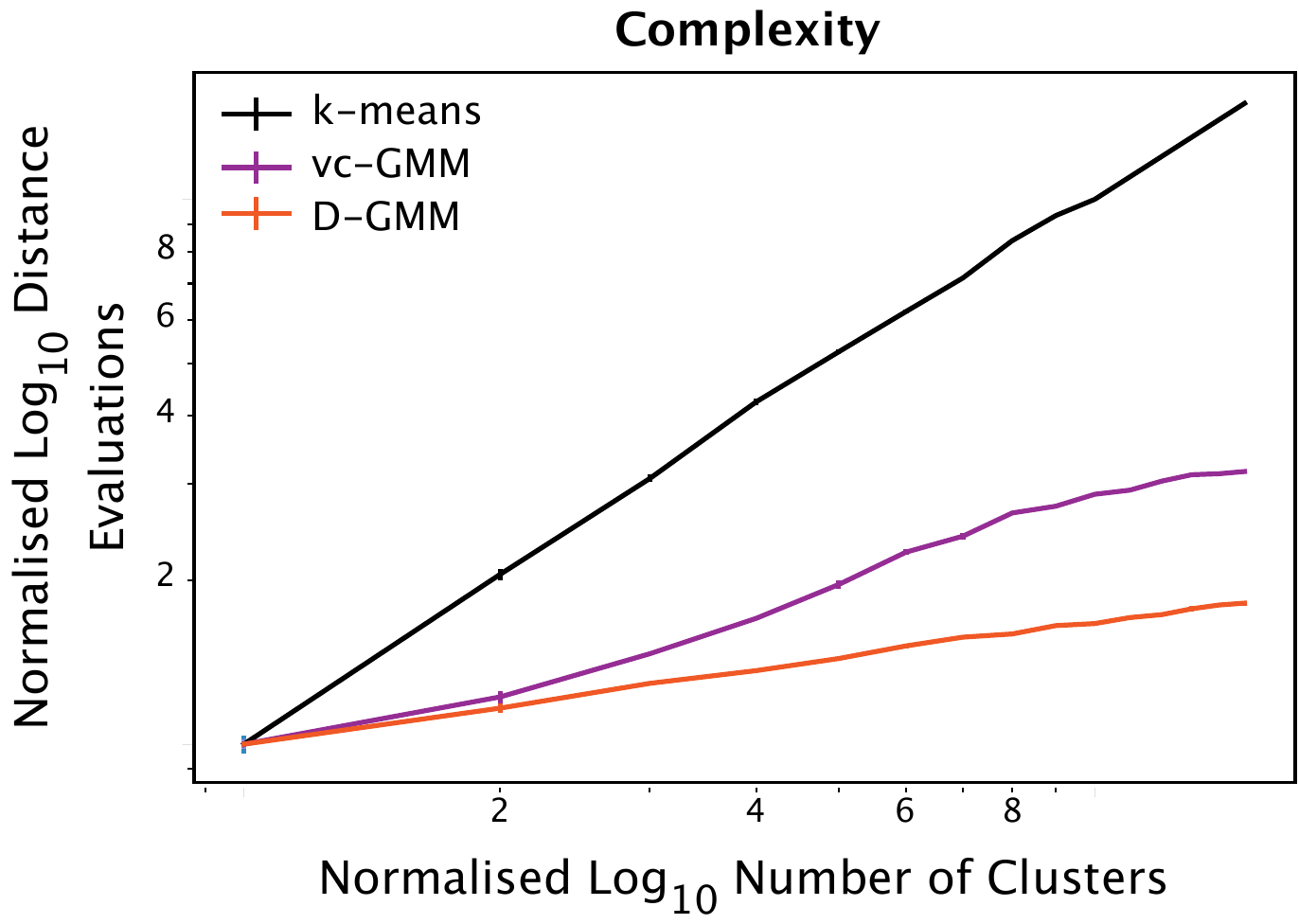}
    \includegraphics[width=0.2\textwidth]{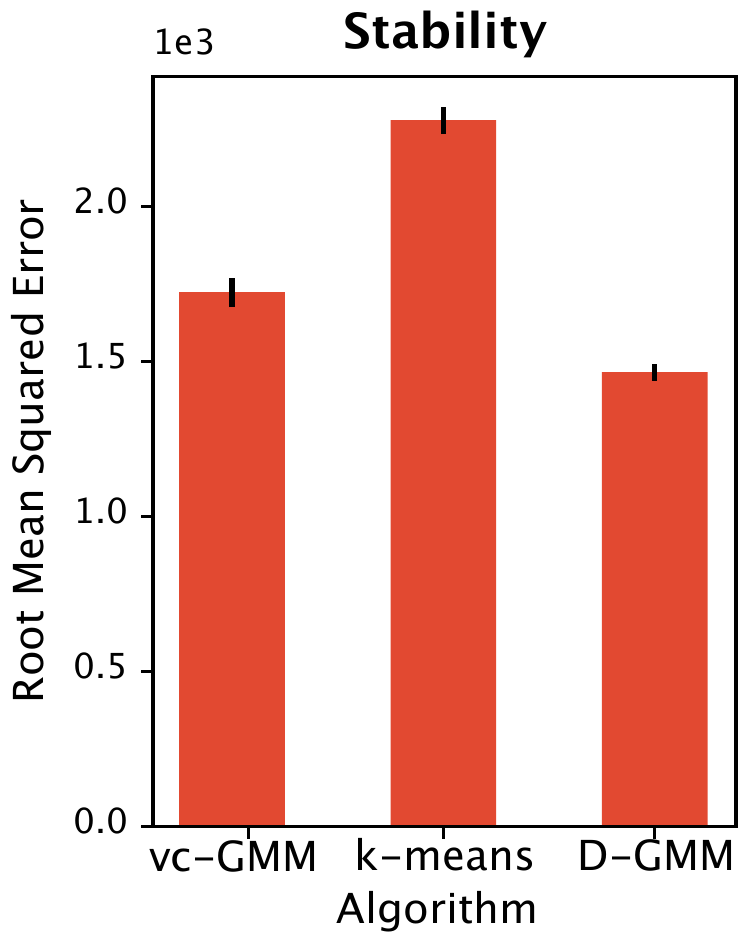} 
    \caption{{\bf Operations complexity and stability on CIFAR-10 over $\mathbf{100}$ trials. (right)} RMSE between the learned centres ($M=500$). {\bf(left)} Normalised log distance evaluations for an increasing $M$, with $100$ trials for each $M$. {\bf (both)} error bars denote $1$ standard deviation.}
    \label{fig:complexity}
\end{figure*}

\begin{figure}[h!]
    \centering
    \includegraphics[width=0.4\textwidth]{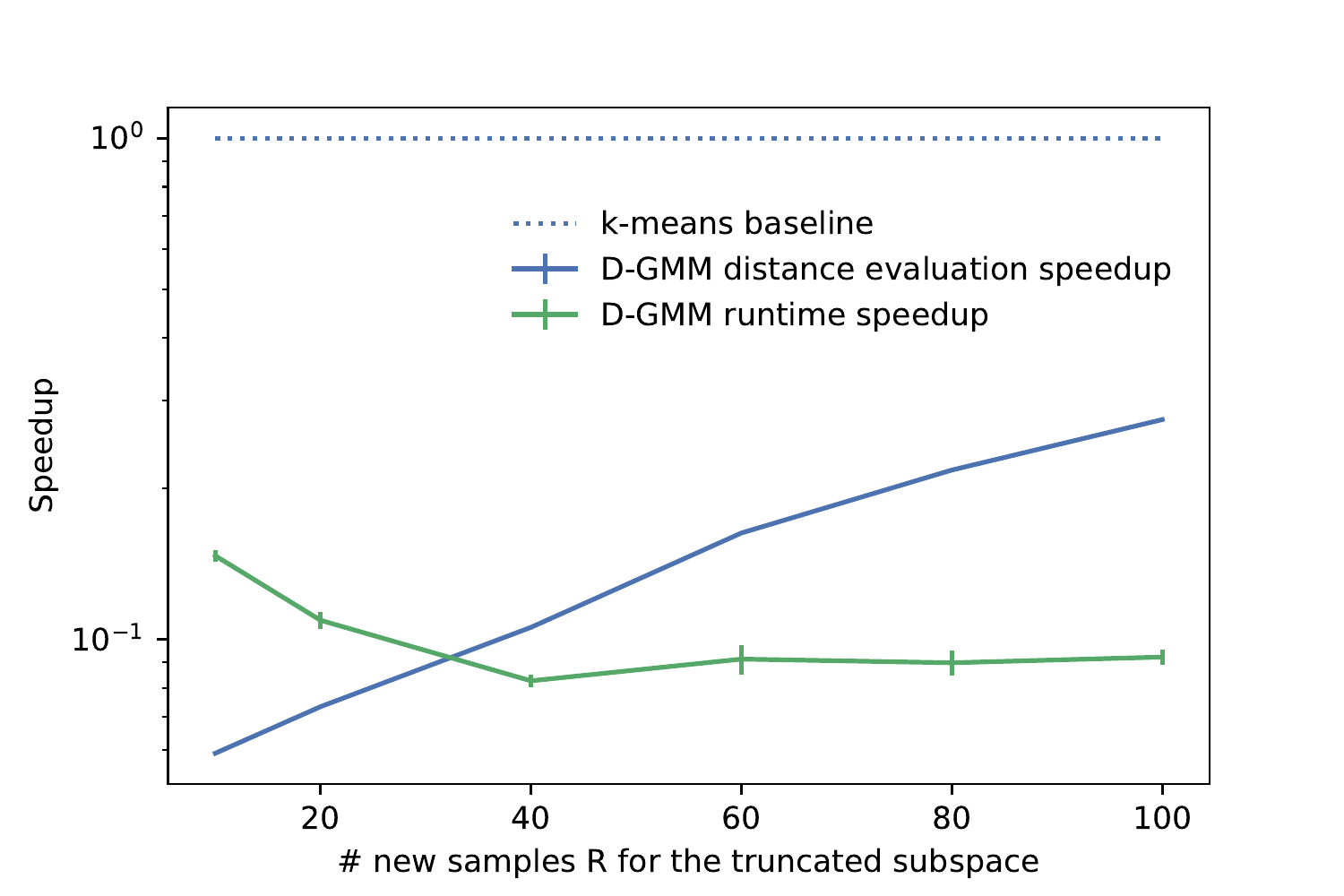}
    \includegraphics[width=0.4\textwidth]{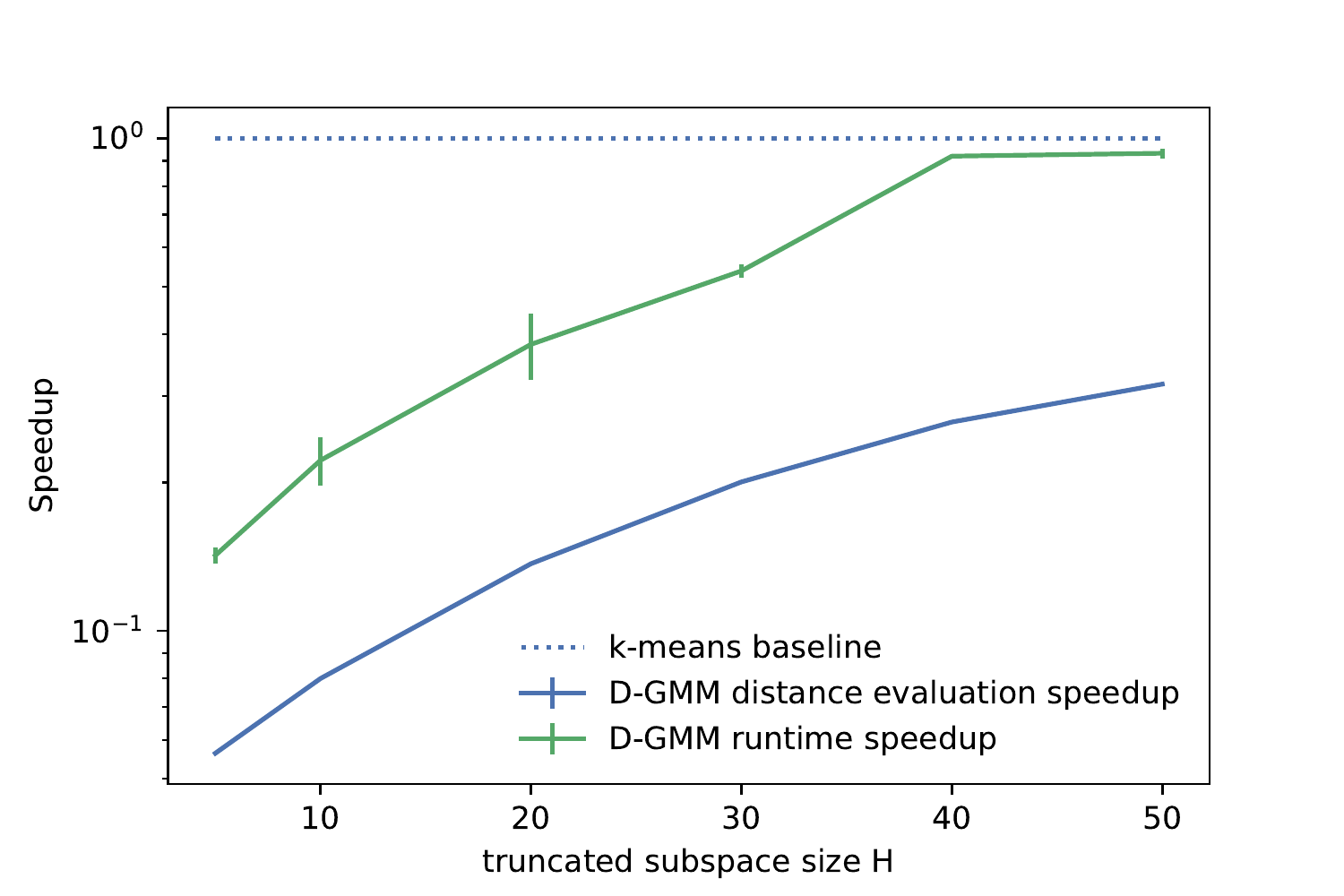}
    \caption{{\bf Hyperparameter search for $R$ (top) and $H$ (bottom).} The blue line shows performance gains with respect to computation. The green line shows performance increase in terms of runtime. The results are relative to the dotted black line which is the exact k-means.}
    \label{fig:hyperparams}
\end{figure}

\paragraph{Stability} We test the ability of the algorithm to recover the same clusters 
using different initialisation. We run the clustering algorithm on CIFAR-10 $100$ times 
with  hyperparameters $M=500, H=5, R=5$, and compare the
recovered centres of every distinct pair of runs  
using the $l_2$-norm between the centres after reshuffling them. The average and standard deviation between all errors are plotted in Fig. \ref{fig:complexity} (right).

% {\color{red} 
\paragraph{Hyperparameter Search}
In Fig. \ref{fig:hyperparams}, we see the effect that the 
hyperparameters $H$ and $R$ have to the optimisation speed-up.
At the top plot, we fix $H=5$ and view the effect $R$ has to the algorithm's number of operations and runtime. We see that reducing the values of $R$ progressively decreases the amount of required operations. In terms of runtime, values lower than $R=40$ introduce a lower speedup. This is due to the fact that we need to perform more iterations and therefore spend more time instantiating samplers than drawing samples from them

To identify the optimal hyperparameter $H$, we fix $R=10$ and observe the effect varying values of $H$ have to the runtime and number of operations of the algorithm. Both runtime and number of operations monotonically reduce with smaller truncated space $H$. Suggesting that caching only a small number of centres is sufficient to efficiently cluster datapoints.

To test the scalability of D-GMM to larger datasets, we use the full ImageNet dataset downsampled to a 64x64 resolution~\cite{deng2009imagenet}. For such a large dataset it was impractical to use k-means as a baseline, as it would not converge in a reasonable time, so we resort to comparing the improvement in number of distance evaluations from initialisation to convergence of the two algorithms. For coreset sizes $2^{13}$, $2^{14}$, $2^{15}$, $2^{16}$, and $2^{17}$ we get a reduction on the number of distance evaluations of a factor of $22$\%, $31$\%, $40$\%, $44$\%, and $49$\% respectively when we use D-GMM compared to vc-GMM.

\section{Discussion}
\label{sec:dis}
We have presented a novel data clustering algorithm. Our algorithm considerably 
increases computational efficiency compared to k-means by calculating the posterior over a data-specific subset of clusters. The subset is iteratively refined by sampling in the neighbourhood of the best performing cluster at each EM iteration.  To identify the neighbourhood of each cluster we propose a similarity matrix based on earlier computed distances between the clusters and datapoints, thus avoiding additional complexity. Furthermore, we implemented lightweight coresets and the AFK-MC\textsuperscript{2} initialisation \cite{bachem2018scalable,bachem2016fast} which are state-of-the-art methods in 
the literature for data pre-processing and GMM centre initialisation 
respectively. We compare our algorithm to vc-GMM \cite{hirschberger2019accelerated,forster2018can} 
which is, to our knowledge, the most efficient GMM algorithm currently available.
In terms of computational complexity, our algorithm is more efficient in 
most cases, improving both with an increasing number of 
datapoints and with an increasing number of clusters compared to vc-GMM.
Furthermore, the advantage in efficiency is complemented by a more stable 
recovery of clusters centres, as demonstrated on the CIFAR-10 database.

It is significant to emphasise that  D-GMM is substantially simpler to intuit and implement compared to vc-GMM.
Arguably, the use of elementary operations on matrix containers is easier to implement than task specific containers for comparison. 
Simplicity of an algorithm is a considerable advantage when 
it comes to communicating and implementing the algorithm 
in different contexts.

% {\color{red}
Our experiments suggest that the bottleneck 
for D-GMM lies with the efficiency of 
low-level operators like calculating exponentials and sampling from discrete distributions. 
Improving these operators could be an interesting future direction in this work, affecting an even larger body of literature.
A key feature of D-GMM that we consider valuable for further development is that reduced computational complexity implies lower energetic demands. 
Therefore, when setting future software development strategies considering considering low-level operators used in clustering algorithms we can take into account the reduced number of operations of D-GMM.
The constraints we introduce on the GMM data 
model are crucial for the D-GMM approximation, 
however, they have an impact on the expressive potential of our algorithm. 
Future work aims at developing the optimisation algorithm in a way 
that would allow training GMMs with fewer constraints  on covariance and prior structure.

To enable further development of this work, we participate in the \href{https://github.com/variational-sublinear-clustering}{``variational sublinear clustering''} organisation with an international team of independent researchers to jointly develop software for efficient probabilistic clustering. 
The D-GMM implementation, found in the supplementary material, will be contributed to this organisation, under an open source license, and further developed through a collaboration with a larger pool of researchers.

In conclusion, we find that D-GMM sets the state-of-the-art in terms of efficiency and stability for GMM-based clustering. There is room for improvement in terms of optimisation of low-level operators and loosening the GMM constraints. A long-term plan to develop and popularise efficient clustering is under way.

\section*{Acknowledgement}
We would like to thank J\"org L\"ucke for his valuable insight during the preparation of this manuscript.
This work was partially supported by French state funds managed within the “Plan Investissements d’Avenir” by the ANR (reference ANR-10-IAHU-02).

\bibliography{references}
\bibliographystyle{plain}

%%%%%%%%%%%%%%%%%%%%%%%%%%%%%%%%%%%%%%%%%%%%%%%%%%%%%%%%%%%%

% \newpage

% \section*{A sampling-based approach for efficient clustering in large datasets -- Supplementary material}
% \label{sec:dis}
% \input{chapters/appendix}
\end{document}